\documentclass[sigconf]{acmart}
\renewcommand\footnotetextcopyrightpermission[1]{} 
\usepackage{listings}

\def\BibTeX{{\rm B\kern-.05em{\sc i\kern-.025em b}\kern-.08emT\kern-.1667em\lower.7ex\hbox{E}\kern-.125emX}}
    
\copyrightyear{2019}
\acmYear{2019}
\setcopyright{acmlicensed}
\acmConference[KDD '19]{KDD '19: Symposium on foo}{August 04--08, 2019}{Anchorage, AK}

\begin{document}

%
\title{Beta Survival Models}

%

\author{David Hubbard, Beno{\^\i}t Rostykus, Yves Raimond, Tony Jebara}
\email{{dhubbard,brostykus,tjebara,yraimond}@netflix.com}
\affiliation{%
  \institution{Netflix}
  \streetaddress{100 Albright Way}
  \city{Los Gatos}
  \state{CA}
  \postcode{95032}
}




%
\renewcommand{\shortauthors}{Hubbard, Rostykus, Raimond and Jebara.}

%
\begin{abstract}
This article analyzes the problem of estimating the time until an event occurs, also known as survival modeling. We observe through substantial experiments on large real-world datasets and use-cases that populations are largely heterogeneous. Sub-populations have different mean and variance in their survival rates requiring flexible models that capture  heterogeneity. We leverage a classical extension of the logistic function into the survival setting to characterize unobserved heterogeneity using the beta distribution. This yields insights into the geometry of the problem as well as  efficient estimation methods for linear, tree and neural network models that adjust the beta distribution based on observed covariates. We also show that the additional information captured by the beta distribution leads to interesting ranking implications as we determine who is most-at-risk. We show theoretically that the ranking is variable as we forecast forward in time and prove that pairwise comparisons of survival remain transitive. Empirical results using large-scale datasets across two use-cases (online conversions and retention modeling), demonstrate the competitiveness of the method. The simplicity of the method and its ability to capture skew in the data makes it a viable alternative to standard techniques particularly when we are interested in the time to event and when the underlying probabilities are heterogeneous.
\end{abstract}

%
%

%

%

%
\maketitle

\section{Introduction}

Survival modeling, customer lifetime value \cite{gupta2006} and product ranking \cite{RudinEtAl12,ChangEtAl2012} are of practical interest when we want to estimate time until a specific event occurs or rank items to estimate which will encounter the event first. Traditionally leveraged in medical applications, today survival regression is extensively used in large-scale business settings such as predicting time to conversion in online advertising and predicting retention (or churn) in subscription services.  Standard survival regression involves a maximum likelihood estimation problem over a specified continuous distribution of the time until event (exponential for the Accelerated Failure Time model \cite{kalbfleisch2011statistical}) or of the hazard function (in the case of Cox Proportional Hazards \cite{cox1972regression}).  In practice, time to event problems are often converted to classification problems by choosing a fixed time horizon which is appropriate for the application at hand.  One then has to balance training the model on recent data against a longer labeling horizon which might be more desirable. Survival models avoid this trade-off by relying on right-censoring. This maps it to missing data problem where not all events are observed due to the horizon over which the data is collected. 

There is evidence (which will be further demonstrated in this article) of the importance of heterogeneity in a variety of real-world time to event datasets. Heterogeneity indicates that items in a data-set have different survival means and variances. For instance, heterogeneity in a retention modeling context would be that as time increases the customers with the highest probability to retain are the ones which still remain in the dataset. Without considering this effect, it might appear that the baseline retention probability has increased over time when in fact the first order effect is that there is a mover/stayer bias. Thus methods which don't consider multiple decision points can fail to adequately account for this effect and thus fall victim to the so called ruse of heterogeneity \cite{vaupel1985}. 

Consider the following example inspired by Porath \cite{ben1973labor} where we have 2 groups of customers, one in which the customers have a retention probability of $0.5$ and in the other the customers are uniformly split between retention probabilities of either $1.0$ or $0.0$.  In this case after having observed only one decision point we would observe the retention probabilities of the two groups to be identical. However, if we consider multiple decision points it becomes clear that the latter population has a much higher long term retention rate because some customers therein retain to infinity.  In order to capture this unobserved heterogeneity we need a distribution that is flexible enough to capture these dynamics and ideally as simple as possible. To that end we posit a beta prior on our instantaneous event probabilities. The beta has only 2 parameters, yet is flexible enough to capture right/left skewed, U-shaped, or normal distributions.
\begin{figure}
  \includegraphics[width=\linewidth]{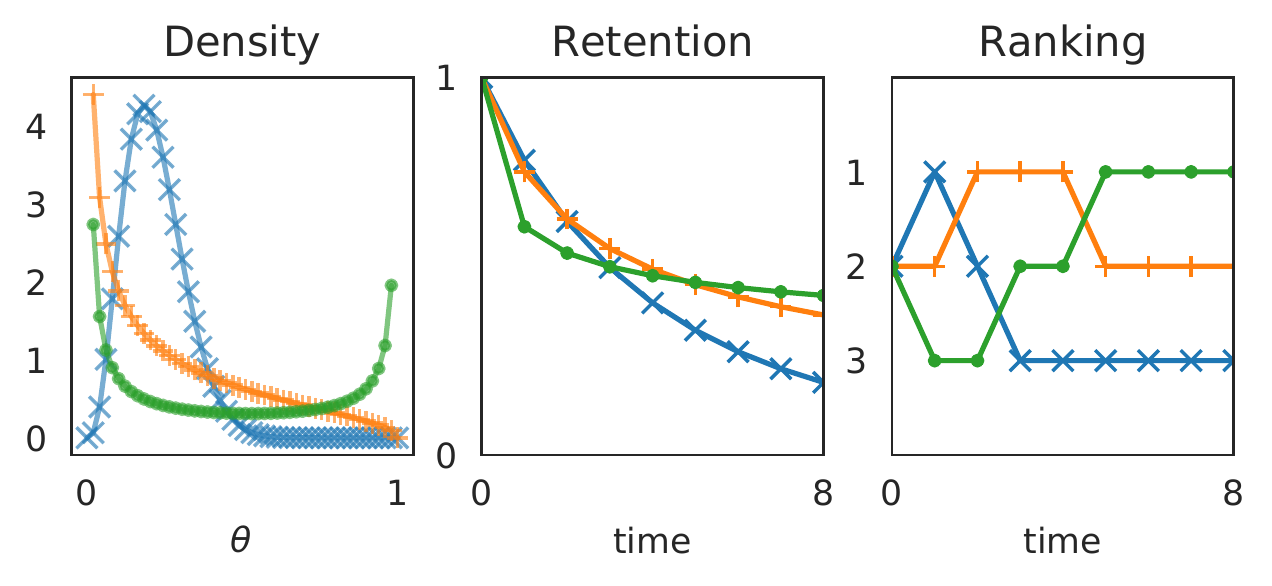}
   \caption{Heterogeneity gives rise to different survival distributions and rankings.}
  \Description{showing the way different densities evolve over time, and how in the end we can improve our rankings as a result of learning this information.}
  \label{fig:teaser_small}
\end{figure}

Consider the example in Figure~\ref{fig:teaser_small}. A data-set contains three heterogeneous items (green dots, orange plus and blue cross). These items are each characterized by beta distributions (left panel). At each time period, each item samples a Bernoulli distributed coin from its beta distribution and flips it to determine if the item will retain. In the middle panel, we see the retention of the items over time and in the right-mode panel we see the ranking of the items over time. Even though the items are sampling from fixed beta distributions, the ranking of which item is most at risk over time changes. Thus, a stationary set of beta distributions lead to non-stationary survival and rankings. Such nuance cannot be captured by summarizing each item with only a point-estimate of survival (as opposed to a 2-parameter beta distribution).

Due to the discrete and repeat nature of the decision problems over time, we leverage a geometric hypothesis to recover survival distributions. We estimate the parameters of this model via an empirical Bayes method which can be efficiently implemented through the use of analytical solutions to the underlying integrals.  This model termed the beta-logistic was first introduced by Heckman and Willis \cite{Heckman1977}, and was also studied by Fader and Hardie \cite{fader2007project}.  We find that in practice this model fits the discrete decision data quite well, and that it allows for accurate projections of future decision points. 

We extend the beta-logistic model to the case of large-scale trees or neural-network models that adjust the beta distribution given input covariates. These leverage the use of recurrence relationships to efficiently compute the gradient.  Through the beta prior underpinning the model, we show empirically that the beta-logistic is able to model a wide range of heterogeneous behaviors that other survival or binary classification models fail to capture, especially in the presence of skew.  As we will see, the beta-logistic model outperforms a typical binary logistic regression in real-world examples, and provides tighter estimated posteriors compared to a typical Laplace approximation. 

We also present theoretical results on ranking with beta distributions. We show that pairwise comparisons between beta distributions can be summarized by the median of the two distributions. This makes ranking with beta distributions a provably transitive problem (pairwise distribution comparisons are generally non-transitive \cite{Savage1994}). Therefore, using the medians of beta distributions allows us to rank the entire population to see which subscribers or which items are most-at-risk. The results are then extended to the case where we rank items across multiple time horizons by approximating the evolution of the survival distribution over time as a product of beta distributions as in \cite{fan1991distribution}. Thus we obtain a consistent ranking of items which evolves over time and is summarized by medians (rather than means) to improve the accuracy of ranking who is most-at-risk. 

This paper is organized as follows.  We first show the beta-logistic derivation as well as reference the recursion formulas which make the computation efficient.  We also make brief remarks about convexity and observe that in practice we rarely encounter convergence issues.  We then present several simulated examples to help motivate the discussion. This is followed by an empirical performance evaluation of the various models across three large real-world datasets: a sparse online conversion dataset and two proprietary datasets from a popular video streaming service involving subscription and viewing behaviors. In all the examples, the beta-logistic outperforms other baseline methods, and seems to perform better in practice regardless of how many attributes are considered. Even though there will always be unobserved variations between individuals that influence the time-to-event, this empirical evidence is consistent.

\section{The Beta-Logistic for Survival Regression}

\subsection{Model derivation}
Denote by $(x_i, t_i, c_i) \in \mathbb{R}^d \times \mathbb{N}\times \{0, 1\}$ a dataset where $x_i$ are covariates, $t_i$ is the discrete time to event for an observed (i.e. uncensored) datapoint ($c_i = 0$) and $t_i$ is the right-censoring time for a datapoint for which the event of interest hasn't happened yet ($c_i = 1)$.  A traditional survival model would posit a parametric distribution $p(T|x)$ and then try to maximize the following empirical likelihood over a class of functions $f$:
\begin{equation*}
L = \prod_{\forall i, c_i = 0} \mathbb{P}\left(T=t_i| f(x_i)\right)\prod_{\forall i, c_i = 1} \mathbb{P}\left(T > t_i | f(x_i)\right). 
\end{equation*}
Unfortunately, unless constrained to a few popular distributions such as the exponential one, the maximization of such a quantity is usually intractable for most classes of functions $f(x)$. 

Let us instead assume that at each discrete decision point, a customer decides to retain with some (point-estimate) probability $1 -\theta$ where $\theta$ is some function of the covariates $x$. Then we further assume that the instantaneous event probability at decision point $t$ is characterized by a shifted geometric distribution as follows:
\begin{equation*}
\mathbb{P}(T=t|\theta) = \theta (1-\theta)^{t-1}, ~
where ~ \theta \in \left[0,1\right].
\end{equation*}
This then gives the following survival equation:
\begin{equation}
\mathbb{P}(T > t|\theta) = 1-\sum_{i=1}^{t}\mathbb{P}(T=i|\theta).
\label{eqn:survival}
\end{equation}
This geometric assumption follows from the discrete nature of the decisions customers need to make in a subscription service, when continuing to next episodes of a show, etc. It admits a a simple and straightforward survival estimate that we can also use to project beyond our observed time horizon.  Now in order to capture the heterogeneity in the data, we can instead assume that $\theta$ follows a conditional beta prior $(\mathbb{B})$ as opposed to being a point-estimate as follows:
\begin{equation*}
f(\theta|\alpha(x),\beta(x))=\frac{\theta^{\alpha(x)-1}(1-\theta)^{\beta(x)-1}}{B(\alpha(x),\beta(x))}
\end{equation*}
where $\alpha(x)$ and $\beta(x)$ are some arbitrary positive functions of covariates (e.g. measurements that characterize a specific customer or a specific item within the population). 

Consider the Empirical Bayes method \cite{gelman2013bayesian} (also called Type-II Maximum Likelihood Estimation \cite{berger2013statistical}) as an estimator for $\alpha(x)$ and $\beta(x)$ given the data:
\begin{equation*}
    \max_{\alpha, \beta} L(\alpha, \beta)
\end{equation*}
where
\begin{equation}\label{eq:likelihood_type2}
L(\alpha, \beta) = \prod_{\forall i,c_i = 0} \mathbb{P}\left(T=t_i| \alpha(x_i),\beta(x_i)\right)\prod_{\forall i, c_i = 1} \mathbb{P}\left(T > t_i | \alpha(x_i),\beta(x_i)\right).
\end{equation}
Using the marginal likelihood function we obtain:
\begin{equation*}
\mathbb{P}(T | \alpha(x), \beta(x)) = \int_{0}^{1}f(\theta|\alpha(x),\beta(x))\mathbb{P}(T | \theta)d\theta.
\end{equation*}
As we will see in the next section, a key property of the beta-logistic model is that it makes the maximization of Equation~\ref{eq:likelihood_type2} tractable. 

Since $\alpha$ and $\beta$ have to be positive to define valid beta-distributions, we use an exponential reparameterization and aim to estimate functions $a(x)$ and $b(x)$ such that:
\begin{equation*}
\alpha(x) = e^{a(x)} ~ and  ~ \beta(x) = e^{b(x)}.
\end{equation*}
Throughout the paper, we will also assume that $a$ and $b$ are twice-differentiable. 

The name \textbf{beta-logistic} for such a model has been coined by \cite{Heckman1977} and studied when the predictors $a(x) = \gamma_a \cdot x$ and $b(x) = \gamma_b \cdot x$ are linear functions.  In this case, at $T=1$ observe that if we want to estimate the mean this reduces to an over-parameterized logistic regression:
\begin{equation*}
\mathbb{P}(T = 1 | \alpha(x), \beta(x)) = \frac{\alpha(x)}{\alpha(x)+\beta(x)} =\frac{1}{1+e^{(\gamma_b-\gamma_a)^\top x}}.
\end{equation*}

\subsection{Algorithm}
We will now consider the general case where $a(x)$ and $b(x)$ are nonlinear functions and could be defined by the last layer of a neural network. Alternatively, they could be generated by a vectored-output Gradient Boosted Regression Tree (GBRT) model.
Using properties of the beta function (see \ref{subsection:recurrence_derivation}), one can show that:
\begin{equation*}\label{eq:t_1_p_formula}
\mathbb{P}(T = 1 | \alpha, \beta) = \frac{\alpha}{\alpha+\beta}
\end{equation*}
and
\begin{equation*}\label{eq:t_1_s_formula}
\mathbb{P}(T > 1 | \alpha, \beta) = \frac{\beta}{\alpha+\beta}.
\end{equation*}
Further, the following recursion formulas hold for $t > 1$:
\begin{equation}\label{eq:recursion_prod_p}
\mathbb{P}(T = t | \alpha, \beta) = \left( \frac{\beta+t-2}{\alpha+\beta+t-1} \right)\mathbb{P}(T = t - 1 | \alpha, \beta)
\end{equation}
and
\begin{equation}\label{eq:recursion_prod_s}
\mathbb{P}(T > t | \alpha, \beta) = \left( \frac{\beta+t-1}{\alpha+\beta+t-1} \right)\mathbb{P}(T > t - 1 | \alpha, \beta).
\end{equation}

If we denote $\ell = -\log L$ as the function we wish to minimize,  Equation~\ref{eq:recursion_prod_p} and Equation~\ref{eq:recursion_prod_s} allow us to derive (see Appendix~\ref{subsection:appendix_gradient}) recurrence relationships for individual terms of $\dfrac{\partial \ell}{\partial \cdot}$ and $\dfrac{\partial^2 \ell}{\partial \cdot^2}$. This makes it possible for example to implement a custom loss gradient and Hessian callbacks in popular GBRT libraries such as XGBoost \cite{chen2016xgboost} and lightGBM \cite{ke2017lightgbm}. In this case, the GBRT models have "vector-output" and predict for every row both $a = \log(\alpha)$ and $b = \log(\beta)$ jointly from a set of covariates, similarly to how the multinomial logit loss is implemented in these libraries. More precisely, choosing a squared loss for the split criterion in each node as is customary, the model will equally weight how well the boosting residuals (gradients) with respect to $a$ and $b$ are regressed on. 

Note that because of the inherent discrete nature of the beta-logistic model, the computational complexity of evaluating its gradient on a given datapoint is proportional to the average value of $t_i$. Therefore, a reasonable time step discretization value needs to be chosen to properly capture the survival dynamics while allowing fast inference.  One can similarly implement this loss in deep learning frameworks. One would typically  explicitly pad the label vectors $t_i$ with zeros up to the max censoring horizon (which would bring average computation per row to $O(\max_i t_i)$ for the mini-batch) so that computation can be expressed through vectorized operations, or via frameworks such as Vectorflow \cite{rostykusvectorflow} or Tensorflow \cite{abadi2016tensorflow} ragged tensors that allow for variable-length inputs to bring the computational cost back to $O(\text{avg}_i t_i)$.

\subsection{Convexity}
For brevity, define $\alpha_i = \alpha(x_i)$ and $\beta_i = \beta(x_i)$.  In the special case where $a(x) = \gamma_{a}\cdot x$ and $b(x) = \gamma_{b}\cdot x$ are linear functions, their second derivative is null and the Hessian of the log-likelihood Equation~\ref{eq:likelihood_type2} is diagonal:
\begin{equation*}\label{eq:convexity_hess_a}
    \frac{\partial^2 \ell}{\partial \gamma_{a,j}^2}  = \gamma_{a,j}^2\alpha_i\sum_{i}\left[ \frac{\beta_i}{\left(\alpha_i + \beta_i\right)^2} + \sum_{u=2}^{t_i}\frac{\beta_i + u - 1}{\left(\alpha_i + \beta_i +u-1\right)^2}\right]
\end{equation*}
and
\begin{equation}\label{eq:convexity_hess_b}
    \begin{aligned}
        \frac{\partial^2 \ell}{\partial \gamma_{b,j}^2}  =  \gamma_{b,j}^2\beta_i\sum_{i}\left[ \frac{\alpha_i}{\left(\alpha_i + \beta_i\right)^2} + \sum_{u=2}^{t_i}k_i\right]
    \end{aligned}
\end{equation}
 where 
\begin{equation*}
   k_i = 
    \begin{cases}
    (\alpha_i+1)\frac{\beta_i^2 - (u-2)(\alpha_i + u - 1)}{\left(\beta_i+u-2\right)^2\left(\alpha_i+\beta_i+u-1\right)^2} & \text{if $c_i = 0$} \\
    \alpha_i\frac{\beta_i^2 - (u-1)(\alpha_i + u-1) }{\left(\beta_i+u-1\right)^2\left(\alpha_i+\beta_i+u-1\right)^2} & \text{otherwise.} 
    \end{cases}     
\end{equation*}
   

We see that the log-likelihood of the shifted-beta geometric model is always convex in $\alpha$ when $a$ is linear.  Further we can see that when all points are observed (no censoring), and the maximum horizon is $T=2$ then Equation~\ref{eq:convexity_hess_b} is also convex in $b$.  

Subsequent terms are not convex, however, but despite that in practice we do not encounter significant convexity issues (e.g. local minima and saddle points).  It seems likely that in practice the convex terms of the likelihood dominate the non-convex terms. Note once again that there is generally no global convexity of the objective function.

\section{Ranking with the Beta-Logistic}
Given $n$ beta distributions another relevant question for the business is how do we rank them from best to worst? This is crucial if, for instance, products need to be ranked to prioritize maintenance as in \cite{RudinEtAl12}. For example, we may be interested in ranking beta distributions for news articles where we are estimating the probability that an article will be clicked on. Or for a video subscription service we could have $n$ beta distributions over $n$ titles each of which represents the probability that this title will be watched to the final episode (title survival). How do we rank the $n$ titles (articles) from most watchable (readable) to least watchable? Let us abstractly view both problems as interchangeable where we are ranking items.

We have two items ($u$ and $v$), each with their own $\alpha,\beta$ parameters that define a beta distribution. Each beta distribution samples a coin with a probability $1-\theta$ of heads (retaining) and $\theta$ for tails (churning out). In the first time step, item $v$ retains less than $u$ if it has a higher coin flip probability, e.g. $p(\theta_v > \theta_u)>0.5 $ or the probability is larger than 50\%. In the case of integer parameters, this is given by the following integral:
\begin{equation*}
\begin{split}
p(\theta_v > \theta_u) & = \int_{\theta_u=0}^1 \int_{\theta_v=\theta_u}^1   \frac{1}{B(\alpha_u,\alpha_v)} \theta_u^{\alpha_u-1} (1-\theta_u)^{\beta_u-1}\cdot  \\ &  \frac{1}{B(\alpha_v,\alpha_v)} \theta_v^{\alpha_v-1} (1-\theta_v)^{\beta_v-1}  d\theta_u d\theta_v
\end{split}
\end{equation*}
which simplifies \cite{miller_2015} into
\begin{equation}\label{eq:ranking_proba_formula}
p(\theta_v > \theta_u) = \sum_{i=0}^{\alpha_v-1} \frac{B(\alpha_u+i,\beta_u+\beta_v) } { (\beta_v + i) B(1+i,\beta_v)B(\alpha_u,\beta_u) }.
\end{equation}

If the quantity is larger than 0.5, then item $u$ retains less than item $v$ in the first time step. It turns out, under thousands of simulations, the less likely survivor is also the one with the larger median $\theta$. Figure~\ref{fig:medianVprobability} shows a scatter plot as we randomly compare pairs of beta distributions. It is easy to see that the difference in their medians agrees with the probability $p(\theta_v > \theta_u)$. So, instead of using a complicated formula to test $p(\theta_v > \theta_u)$, we just need to compare the medians via the inverse incomplete beta function (\texttt{betaincinv}) denoted by $I^{-1}()$ and see if $I^{-1}(0.5,\alpha_u,\beta_u) > I^{-1}(0.5,\alpha_v,\beta_v)$. A proof of this is below.

\begin{figure}
  \includegraphics[width=\linewidth]{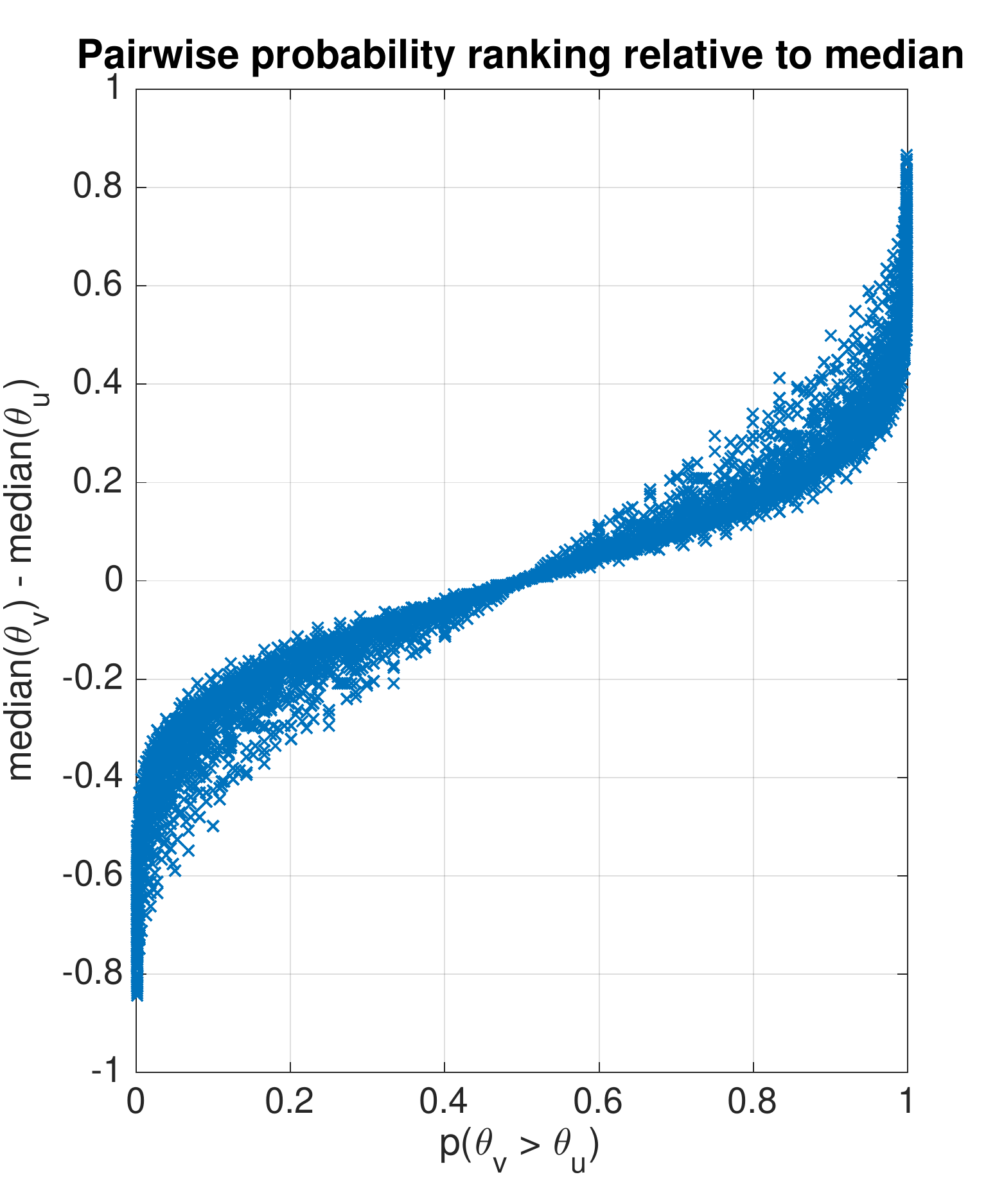}
  \caption{The medians of beta distributions are consistent with the pairwise probabilistic ranking of beta distributions.}
  \Description{medians are consistent with pairwise ranking.}
  \label{fig:medianVprobability}
\end{figure}

For later time steps, we will leverage a geometric assumption but applied to distributions rather than point estimates. The item which retains longer is the one with the lower product of repeated coin flip probabilities, i.e. $p(\theta_v^t > \theta_u^t)$. In that case, the beta distributions get modified by taking them to powers. Sometimes, the product of beta-distributed random variables is beta-distributed \cite{fan1991distribution} but other times it is {\em almost} beta distributed. In general, we can easily compute the mean and variance of this power-beta distribution to find a reasonable beta approximation to it with its own $\alpha,\beta$ parameters. This is done by leveraging the derivation of \cite{fan1991distribution} as follows. Assume we have a beta distribution $p(\theta|\alpha,\beta)$ and define the new random variable $z=(1-\theta)*\theta^{t-1}$. We can then derive the first moment as
\begin{equation*}
S=E_{p(z)}[z] = \left ( \frac{\beta}{\alpha+\beta} \right )\left ( \frac{\alpha}{\alpha+\beta} \right )^{t-1}
\end{equation*}
and the second moment as
\begin{equation*}
T=E_{p(z)}[z^2] = \left(\frac{\alpha(\alpha+1)}{(\alpha+\beta)(\alpha+\beta+1)} \right )\left(\frac{\beta(\beta+1)}{(\alpha+\beta)(\alpha+\beta+1)} \right )^{t-1}.
\end{equation*}
Then, we approximate the distribution for the future event probabilities $p(z)$ by a beta distribution where 
\begin{equation}
\begin{aligned}
\hat{\alpha} & = \frac{(S-T)S}{T-S^2}  \\
\hat{\beta} & = \frac{(S-T)(1-S)}{(T-S^2)}.
\label{eqn:evolution_of_beta_approximation}
\end{aligned}
\end{equation}
Therefore, in order to compare who is more likely to survive in future horizons, we can combine Equation~\ref{eqn:evolution_of_beta_approximation} and Equation~\ref{eqn:survival} to find the median of the approximated future survival distribution.

\begin{theorem}
\label{thm:maintheorem}
For random variables $\theta_u \sim {\rm Beta}(\alpha_u,\beta_u)$ and $\theta_v \sim {\rm Beta}(\alpha_v,\beta_v)$ for $\alpha_u,\alpha_v,\beta_u,\beta_v \in \mathbb{N}$,
$p(\theta_v > \theta_u) > 0.5$ if and only if $I^{-1}(0.5,\alpha_v,\beta_v) > I^{-1}(0.5,\alpha_u,\beta_u)$  (the median of $\theta_v$ is larger than the median of $\theta_u$). 
\end{theorem}

\begin{proof}
We first prove that the median gives the correct winner under simplifying assumptions when both beta distributions have the same $\alpha$ or the same $\beta$.

First consider when the distributions have the same $\alpha_u=\alpha_v=\alpha$ and different $\beta_u$ and $\beta_v$. In that case, Equation~\ref{eq:ranking_proba_formula} simplifies to
\begin{equation}
p(\theta_v > \theta_u) = \sum_{i=0}^{\alpha_v-1} \frac{ B(\alpha_u+i,\beta_u+\beta_v) } { (\beta_v + i) B(1+i,\beta_v) B(\alpha,\beta_u) }.
\end{equation}
The formulas for $p(\theta_v > \theta_u)$ and $p(\theta_u > \theta_v) $ only differ in their denominators. Then, if $\beta_v > \beta_u$ it is easy to show that
\begin{equation}
    (\beta_v + i) B(1+i,\beta_v) B(\alpha,\beta_u) >  (\beta_u + i) B(1+i,\beta_u) B(\alpha,\beta_v).
\end{equation}
Therefore, $p(\theta_v > \theta_u) > p(\theta_u > \theta_v)$ if and only if $\beta_v > \beta_u$. 

Similarly, if $\beta_v > \beta_u$, the medians satisfy $I^{-1}(0.5,\alpha_v,\beta_v) < I^{-1}(0.5,\alpha_u,\beta_u)$. This is true since, if all else is equal, increasing the $\beta$ parameter reduces the median of a beta distribution. Therefore, for $\alpha_u=\alpha_v$, the median ordering is always consistent with the probability test.

An analogous derivation holds when the two distributions have the same $\beta_u=\beta_v$ and different $\alpha_u$ and $\alpha_v$. This is obtained by using the property $p(\theta|\alpha,\beta)=p(1-\theta|\beta,\alpha)$. Therefore, for $\beta_u=\beta_v$, the median ordering is always consistent with the probability test.

Next we generalize these two statements to show that the median ordering always agrees with the probability test. Consider the situation where $median(\alpha_u,\beta_u)>median(\alpha_u,\beta_v)>median(\alpha_v,\beta_v)$. Due to the scalar nature of the median of the beta distribution, we must have transitivity. We must also have $median(\alpha_u,\beta_u) < median(\alpha_v,\beta_u) < median(\alpha_v,\beta_v)$. Since each pair of inequalities on medians requires that the corresponding statement on the probability tests also holds, the overall statement $median(\alpha_u,\beta_u)<median(\alpha_v,\beta_v)$ must also imply that $p(\theta_v > \theta_u) > p(\theta_u > \theta_v)$.
\end{proof}

Therefore, thanks to Theorem~\ref{thm:maintheorem}, we can safely rank order beta distributions simply by considering their medians. These rankings are not only pairwise consistent but globally consistent. Recall that pairwise ranking of distributions does not always yield globally consistent rankings as  popularly highlighted through the study of nontransitive dice \cite{Savage1994}.  Thus, given any beta distribution at a particular horizon, it is straightforward to determine which items are most at risk through a simple sorting procedure on the medians. Given this approach to ranking beta distributions, we can show the performance of our model-based ranking of users or items by holding out data and evaluating the time to event in terms of the AUC (area under the curve) of the receiver operating characteristic (ROC) curve.

\section{Empirical Results}

\subsection{Synthetic simulations}
We present simulation results for the beta-logistic, and compare them to the logistic model.  We also show that the beta-logistic successfully recovers the posterior for skewed distributions.  In our first simulation we have 3 beta distributions which have very different shapes (see table \ref{table:simulation1} below), but with the same mean (this example is inspired by Fader and Hardie  \cite{fader2018project}).  Here, each simulated customer draws a coin from one of these distributions, and then flips that coin repeatedly until they have an event or we reach a censoring horizon (in this particular case we considered 4 decision points).
\begin{table}[h]
\begin{tabular}{ |p{2cm}||p{1cm}|p{1cm}|p{1cm}|  }
 \hline
 shape & $\alpha$ & $\beta$ & $\mu$\\
 \hline
 normal   & 4.75 & 14.25 &  0.25 \\
 right skewed &  0.5  & 1.50 & 0.25  \\
 u shaped & $0.08\bar{3}$ & 0.25 &  0.25 \\
 \hline  
\end{tabular}
 \caption{Heterogeneous beta distributions with identical means.}
  \label{table:simulation1}
\end{table}

It is trivial to show that the logistic model will do no better than random in this case, because it is not observing the dynamics of the survival distribution which reveal the differing levels of heterogeneity underlying the 3 populations.  If we allow the beta-logistic model to have a dummy variable for each of these cases then it can recover the posterior of each (see Figure~\ref{fig:simulation1_survival}).  This illustrates an important property of the beta-logistic: it  recovers posterior estimates even when the data is very heterogeneous and allows us to fit survival distributions well.

\begin{figure}[h]
  \centering
  \includegraphics[width=\linewidth]{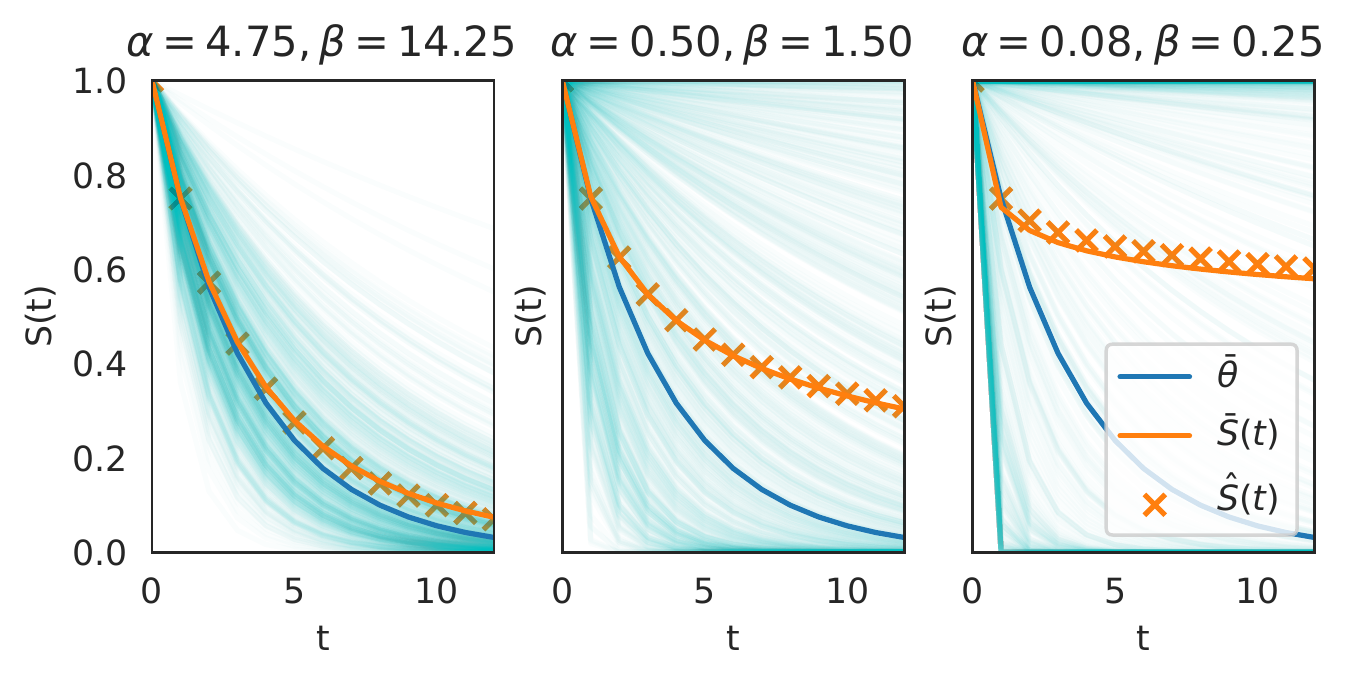}
  \caption{Survival distributions as a function of time as well as an estimate of $\hat{S}(t)$ from the beta-logistic. Using a point-estimate of the mean $\Bar{\theta}$ (as in the logistic model) fails to recover the heterogeneity.}
  \label{fig:simulation1_survival}
  \Description{Simulation 1}
\end{figure}

To create a slightly more realistic simulation, we can include another term which increases the homogeneity linearly in $\alpha$ and we add this as another covariate in the models. We also inject exponential noise into the $\alpha$ and $\beta$ used for our random draws. Now, the logistic model does do better than random when there is homogeneity present (see Figure~\ref{fig:simulation2}), however it still leaves signal behind by not considering the survival distribution. We additionally show results for a one time step beta logistic which performs similarly to the logistic model. However it seems to have a lower variance which perhaps indicates that its posterior estimates are more conservative, a property which will be confirmed in the next set of experiments.  

\begin{figure}[h]
  \centering
  \includegraphics[width=\linewidth]{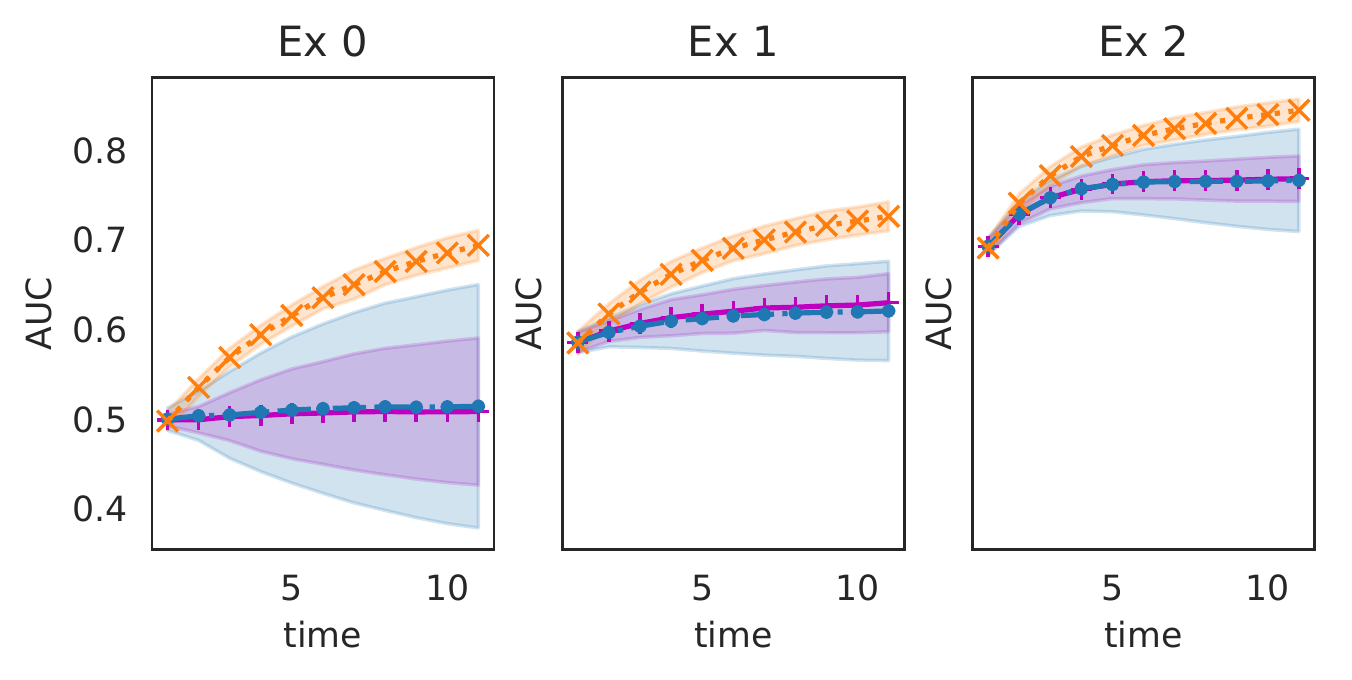}
  \caption{The level of heterogeneity increases from the left panel to the right panel as we add a linear term in $\alpha$. Clearly, the mean of the beta-logistic 1 step (magenta plus), and logistic (cyan dot) are nearly identical, but the beta-logistic (orange cross) considers more survival information and outperforms both even when there is considerable homogeneity.}
  \Description{Simulation 2}
  \label{fig:simulation2}
\end{figure}

\subsection{Online conversions dataset}
\subsubsection{Survival modeling}
We now evaluate the performance of the beta logistic model on a large-scale sparse dataset. We use the Criteo online conversions dataset published alongside \cite{chapelle2014modeling} and publicly available for download\footnote{http://labs.criteo.com/2013/12/conversion-logs-dataset/}. We consider the problem of modeling the distribution of the time between a click event and a conversion event. We will consider a censoring window of 12 hours (61\% of conversions happen within that window). As noted in \cite{chapelle2014modeling}, the exponential distribution fits reasonably well the data so we will compare the beta-logistic model against the exponential distribution (1 parameter) and the Weibull distribution (2 parameters). Since the temporal integration of the beta-logistic model is intrinsically discrete, we consider a time-discretization of 5 minute steps. We also add as baselines 2 logistic models: one trained at a horizon of 5 minutes (the shortest interval), and one trained at a horizon of 12 hours (the largest window). All conditional models are implemented as sparse linear models in Vectorflow \cite{rostykusvectorflow} and trained through stochastic gradient descent. All survival models use an exponential reparameterization of their parameters (since the beta, exponential, and Weibull distributions all require positivity in their parameters). Censored events are down-sampled by a factor of 10x. We use 1M rows for training and 1M (held-out in time) rows for evaluation. The total (covariate) dimensionality of the problem is 102K after one-hot-encoding. Note that covariates are sparse and the overall sparsity of the problem is over 99.98\%. Results are presented in Figure~\ref{fig:criteo_exp1}.

\begin{figure}[h]
  \centering
  \includegraphics[width=\linewidth]{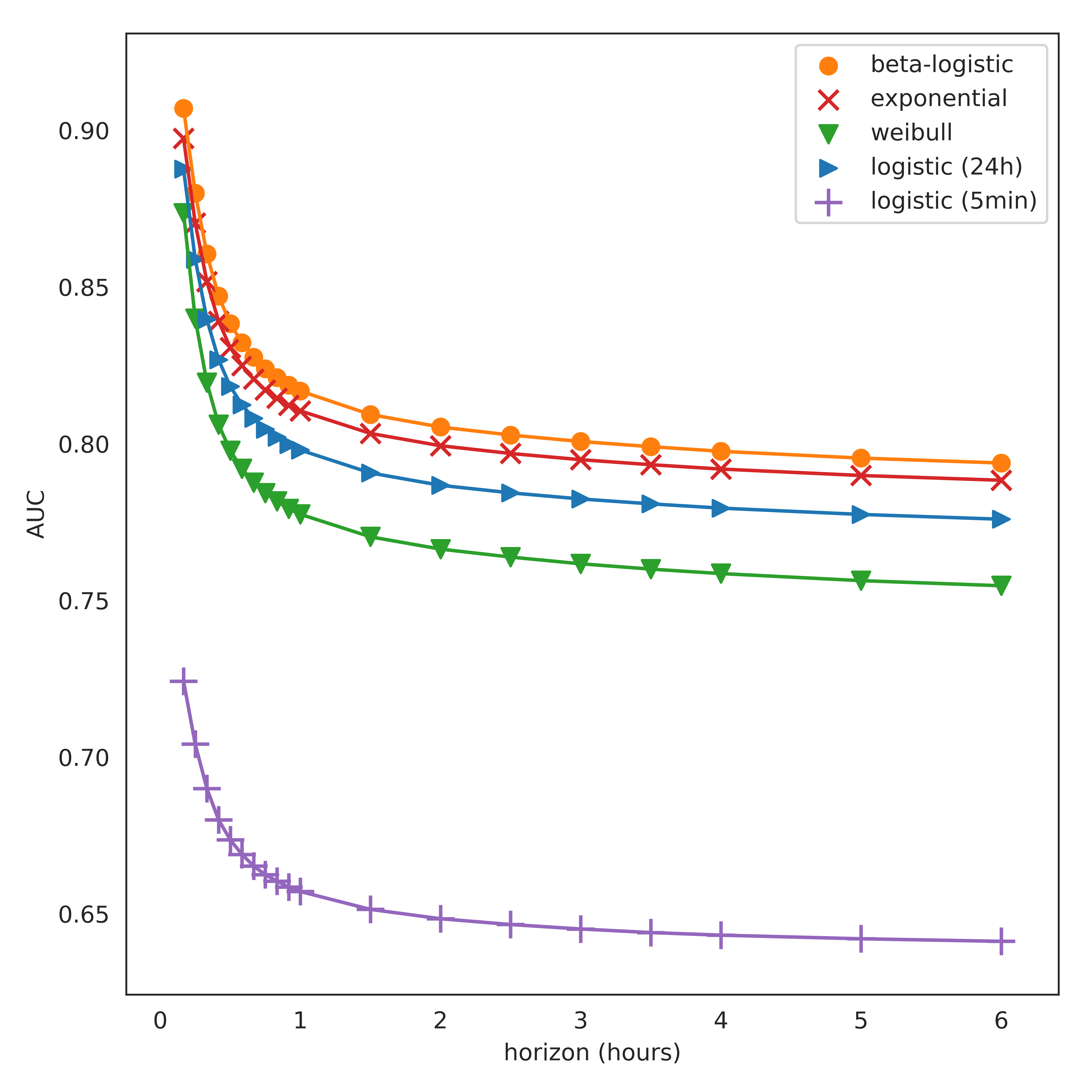}
  \caption{AUC as a function of censoring horizon for the various models considered.}
  \label{fig:criteo_exp1}
  \Description{Online conversion experiment 1}
\end{figure}

The beta-logistic survival model outperforms other baselines at all horizons considered.
Even though it is a 2-parameter distribution, the Weibull model is interestingly performing worse than the exponential survival model and the binary logistic classifier. We hypothesize that this is due to the poor conditioning of its loss function as well as the numerical instabilities during gradient and expectation computation (the latter requires function calls to the gamma function which is numerically difficult to estimate for moderately small values and for large values).

\subsubsection{Posterior size comparison}
We next consider the problem as a binary classification task (did a conversion happen within the specified time window?). It is interesting to compare the confidence interval sizes of various models. For the conditional beta-logistic model, the prediction variance on datapoint $x$ is given by:
\begin{equation*}\label{eq:beta_logistic_variance}
    \text{Var}(x) = \dfrac{\alpha(x)\beta(x)}{\left(\alpha(x)+\beta(x)\right)^2\left(\alpha(x)+\beta(x)+1\right)}.
\end{equation*}
For a logistic model parameterized by $\theta\in\mathbb{R}^d$, a standard way to estimate the confidence of a prediction is through the Laplace approximation of its posterior \cite{mackay2003information}. In the high-dimensional setting, estimating the Hessian or its inverse become impractical tasks (storage cost is $O(d^2)$ and matrix inversion requires $O(d^3)$ computation). In this scenario, it is customary to assume independence of the posterior coordinates and restrict the estimation to the diagonal of the Hessian $h = \dfrac{1}{\sigma^2} \in \mathbf{R}^d$, which reduces both storage and computation costs to $O(d)$. Hence under this assumption, for a given datapoint $x$ the distribution of possible values for the random variable $Y = \theta^T x$ is also Gaussian with parameters:
\begin{equation*}
    \mathcal{N}\left(\sum_i \theta_i x_i, \sum_i \sigma_{i}^2x_i^2 \right).
\end{equation*}
If the full Hessian inverse $H^{-1}$ is estimated, then $Y$ is Gaussian with parameters:
\begin{equation*}
    \mathcal{N}\left(\theta \cdot x, x^T \cdot H^{-1}x)\right).
\end{equation*}
When $Y$ is Gaussian, the logistic model prediction
\begin{equation*}
    \mathbb{P}(T=1 \vert x, \theta) = \dfrac{1}{1 + \exp(-Y)}
\end{equation*}
has a distribution for which the variance $v$ can be conveniently approximated. See \cite{li2012unbiased} for various suggested approximations schemes. We chose to apply the following approximation
\begin{equation*}
    v = \Phi\left(\dfrac{\pi\mu/\sqrt{8} -1}{\sqrt{\pi -1 + \pi^2\sigma^2/8}} \right) - \left(1 + \exp(-\mu/\sqrt{1 + \pi\sigma^2/8}) \right)^{-2}.
\end{equation*}

Armed with this estimate for the logistic regression posterior variance, we run the following experiment: we random-project (using Gaussian vectors) the original high-dimensional data into a 50-dimensional space, in which we train beta-logistic classifiers and logistic classifiers at various horizons, using 50k training samples every time. We then compare the average posterior size variance on a held-out dataset containing 50k samples. Holdout AUCs were comparable in this case for both models at all horizons. Two posterior approximations are reported for the logistic model: one using the full Hessian inverse and the other one using only the diagonalized Hessian. Results are reported in Figure~\ref{fig:criteo_exp2}.

\begin{figure}[h]
  \centering
  \includegraphics[width=\linewidth]{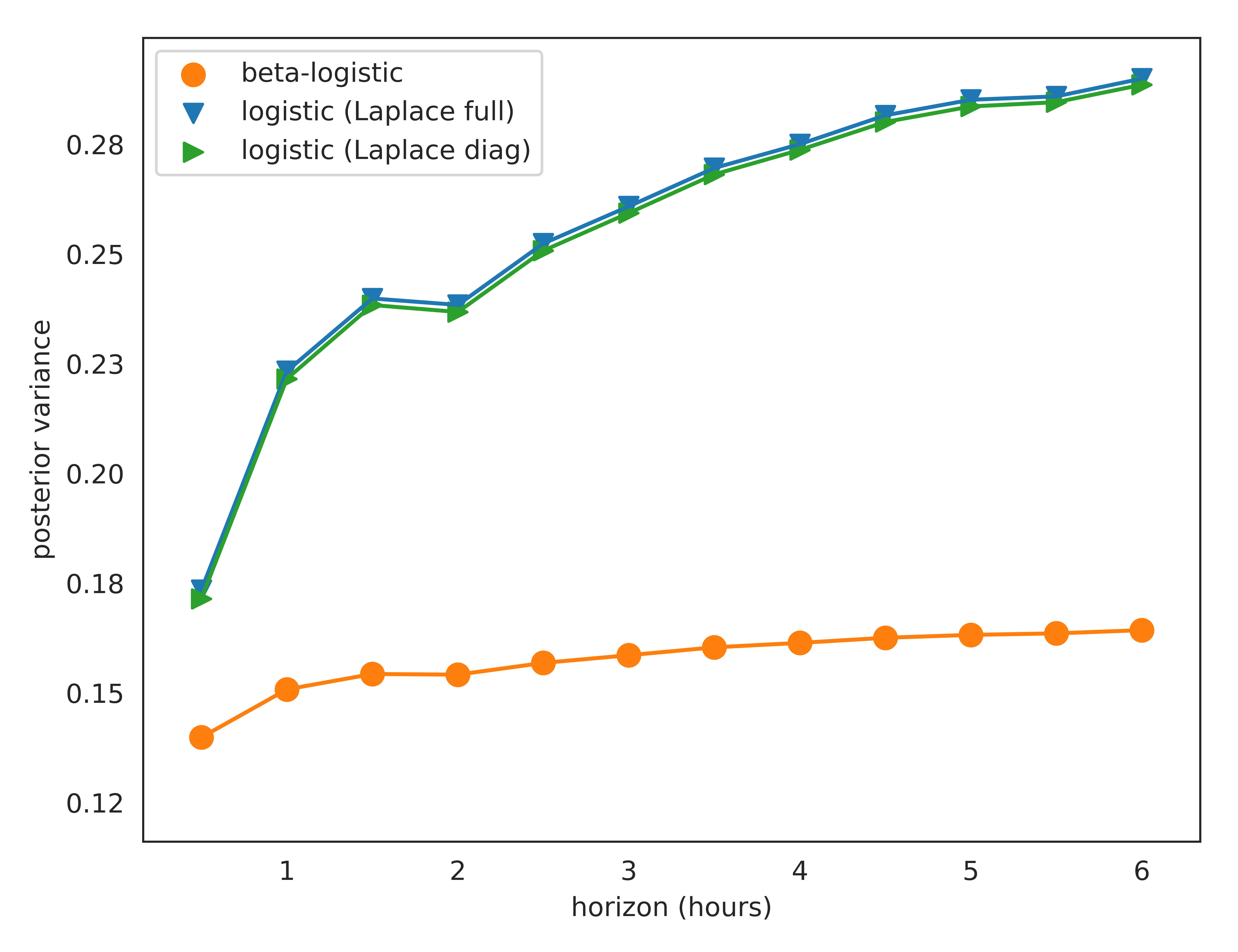}
  \caption{The posterior variance of beta-logistic binary classifiers as well as logistic regressions trained on binary labels datasets with increasing censoring windows.}
  \label{fig:criteo_exp2}
  \Description{Online conversion experiment 2}
\end{figure}

Note that the beta-logistic model produces much smaller uncertainty estimates (between 20\% and 45\% smaller) than the logistic model with Laplace approximation. Furthermore, the growth rate as a function of the horizon of the binary classifier is also smaller for the beta-logistic approach. Also note that the Laplace posterior with diagonal Hessian approximation underestimates the posterior obtained using the full Hessian. Gaussian posteriors are  obviously unable to appropriately model data skew. 

This empirical result is arguably clear evidence of the superiority of the posteriors generated by the beta-logistic model over a standard Laplace approximation estimate layered onto a logistic regression model posterior. The beta-logistic posterior is also much cheaper to recover in terms of computational and storage costs.
This also suggests that the beta-logistic model could be a viable alternative to standard techniques of explore-exploit models in binary classification settings.

\subsection{Video streaming subscription dataset}
\subsubsection{Retention modeling}
We now study the problem of modeling customer retention for a subscription business. We leverage a proprietary dataset from a popular video streaming service. In a subscription setting when a customer chooses not to renew the service, the outcome is explicitly observed and logged. From a practical perspective, it is obviously preferable and meaningful when a customer's tenure with the service is n months rather than 1 month. It is also clearly valuable to be able to estimate and project that tenure accurately across different cohorts of customers. In this particular case the cohorts are highly heterogeneous as shown in  Figure~\ref{fig:netflix_beta}.

\begin{figure}[h]
  \centering
  \includegraphics[width=\linewidth]{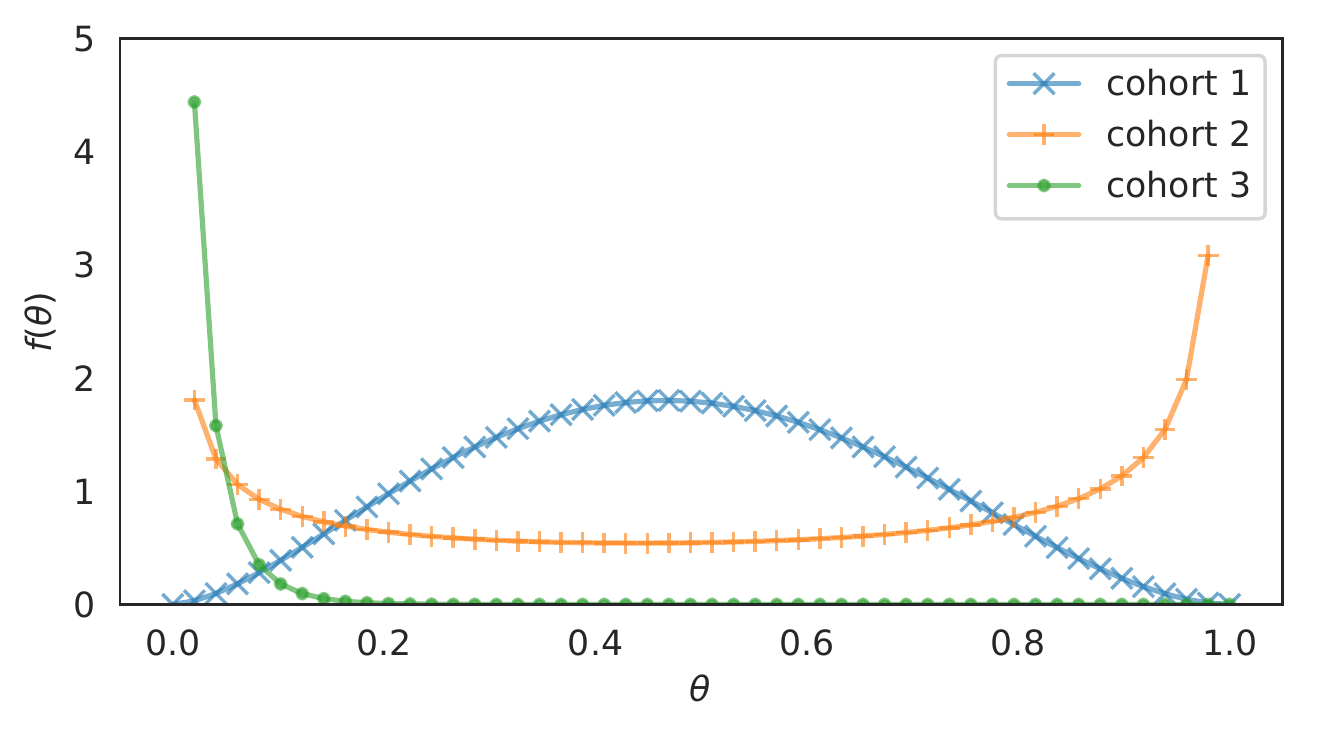}
  \caption{Estimated churn probabilities for 3 different cohorts. The large variations in the shape of the fitted distributions motivate the use of a beta prior on the conditional churn probability.}
  \Description{Subscription Retention Model AUC}
  \label{fig:netflix_beta}
\end{figure}

In this example the data set had more than 10M rows and 500 columns. We trained 20 models on bootstraps of the data with a 10x downsample on censored customers. We used 4 discrete decision points to fit the model. Evaluation was done on a subset of 3M rows which was held out from the models, and held out in time as well over an additional 5 decision points (9 total months of data). All models are implemented as GBRTs and estimated using lightGBM.  In Figure~\ref{fig:sub_auc}, we show the evaluation of the models across two cohorts:  one with relatively little data and covariates to describe the customers (which should clearly benefit from modeling the unobserved heterogeneity) and one with much richer data and covariates. Surprisingly even on the rich data set where one might argue there should be considerable homogeneity within any given subset of customers, we still find accuracy improvements by using the beta-logistic over the standard logistic model. This example illustrates how regardless of how many covariates are considered, there is still considerable heterogeneity.
\begin{figure}[h]
  \centering
  \includegraphics[width=\linewidth]{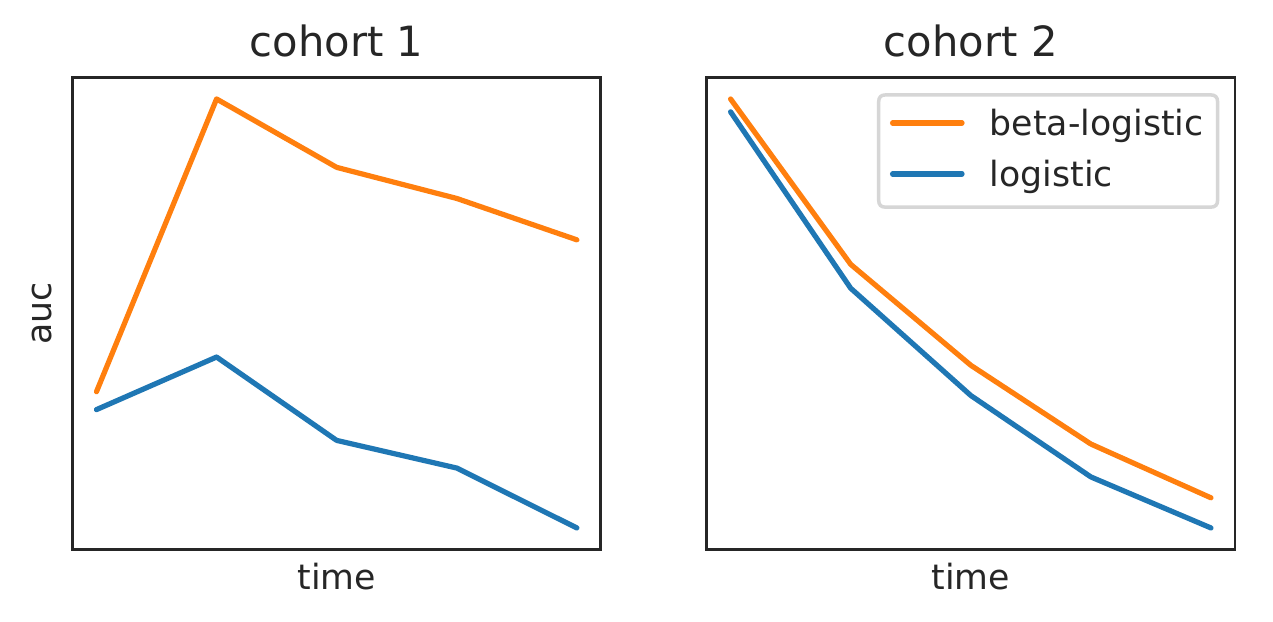}
  \caption{Held-out AUC for two different cohorts of customers.}
  \Description{AUC of the Subscription Retention Model.}
  \label{fig:sub_auc}
\end{figure}
\subsubsection{Retention within shows}
Another problem of importance to a video subscription business is ranking shows that customers are most likely to fully enjoy (i.e. customers watch the shows to completion across all the episodes).  Here we model the distribution of survival of watched episodes of a show conditional on the customer having started the show. 
In Figure~\ref{fig:netflix completion} we compare the performance of the beta-logistic to logistic models at an early (1 episode) horizon and a late horizon (8 episodes).  The dataset contains 2k shows and spans 3M rows and 500 columns. We used a 50/50 train/test split. Model training used bootstrap methods to provide uncertainty estimates.  All models are implemented as GBRTs which were estimated in lightGBM. In early horizons, the beta-logistic model once again provides significant AUC improvements over logistic models trained at either the 1 episode horizon and 8 episode horizon.

\begin{figure}[h]
  \centering
  \includegraphics[width=\linewidth]{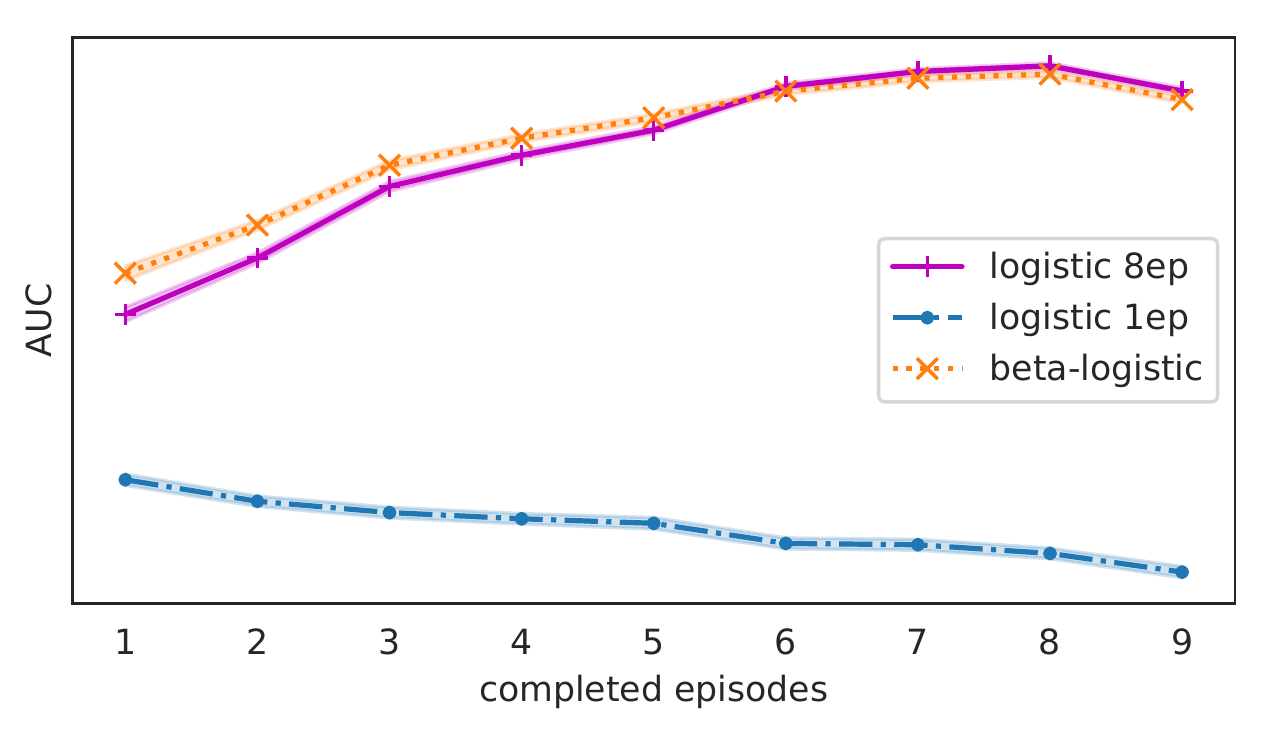}
  \caption{The beta-logistic improves ranking accuracy (in terms of AUC) for early horizons.}
  \label{fig:netflix completion}
\end{figure}

\section{Conclusion}
We noted that heterogeneity in the beta-logistic model can better capture the temporal effects of survival and ranking at multiple horizons. 
We extended the beta-logistic and its maximum likelihood estimation to linear, tree and neural models as well as characterized the convexity and properties of the learning problem. The resulting survival models give survival estimates and provably consistent rankings of who is most-at-risk at multiple time horizons. Empirical results demonstrate that the beta-logistic is an effective model in discrete time to event problems, and improves over common baselines. It seems that in practice regardless of how many attributes are considered there are still unobserved variations between individuals that influence the time to event. Further we demonstrated that we can recover posteriors effectively even when the data is very heterogeneous, and due to the speed and ease of implementation we argue that the beta-logistic is a baseline that should be considered in time to event problems in practice.

In future work, we plan to study the potential use of the beta-logistic in explore-exploit scenarios and as a viable option in reinforcement learning to model long-term consequences from near-term decisions and observations.    

%
\begin{acks}
The authors would like to thank Nikos Vlassis for his helpful comments during the development of this work, and Harald Steck for his thoughtful review and guidance.
\end{acks}

%
\bibliographystyle{ACM-Reference-Format}
\bibliography{sample-base}

%
\begin{appendix}

\section{Beta Logistic Formulas}

\subsection{Recurrence derivation}
\label{subsection:recurrence_derivation}
This derivation is taken from Fader and Hardie \cite{fader2007project} where they use it as a cohort model (also called the shifted beta geometric model) that is not conditional on a covariate vector $x$. 

We do not observe $\theta$, but its expectation given the beta prior (also called marginal likelihood) is given by:
\begin{equation*}
\begin{aligned}
\mathbb{P}(T=t|\alpha,\beta) & = \int_{0}^{1}\theta (1-\theta)^{t-1} \frac{\theta^{\alpha-1}(1-\theta)^{\beta-1}}{B(\alpha,\beta)}d\theta\\
& = \frac{B(\alpha+1,\beta+t-1)}{B(\alpha,\beta)}
\end{aligned}
\end{equation*}
We can write the above as:
\begin{equation*}
\mathbb{P}(T=t|\alpha,\beta) =\frac{\Gamma(\alpha+\beta)*\Gamma(\alpha+1)*\Gamma(\beta+ t - 1)}{\Gamma(\alpha)*\Gamma(\beta)*\Gamma(\alpha+\beta+t)}.   
\end{equation*}
Using the property $\Gamma(z+1) = z\Gamma(z)$ leads to equations (\ref{eq:recursion_prod_p}) and (\ref{eq:recursion_prod_s}), and at $t=1$ we have
\begin{equation*}
\begin{aligned}
\mathbb{P}(T=1|\alpha,\beta) & = \frac{\Gamma(\alpha+\beta)*\Gamma(\alpha+1)*\Gamma(\beta)}{\Gamma(\alpha)*\Gamma(\beta)*\Gamma(\alpha+\beta+1)}\\
\mathbb{P}(T=1|\alpha,\beta) & = \frac{\alpha}{\alpha+\beta}
\end{aligned}
\end{equation*}

\subsection{Gradients}
Note that for machine learning libraries that do not offer symbolic computation and auto-differentiation, taking the $-log$ of equations (\ref{eq:recursion_prod_p}) and (\ref{eq:recursion_prod_s}) and differentiating leads to the following recurrence formulas for the gradient of the loss function on a given data point with respect to the output parameters $a_i$ and $b_i$ of the model considered:
\label{subsection:appendix_gradient}
\begin{equation*}
\begin{aligned}
\frac{\partial \log(\mathbb{P}(T=1))}{\partial a_i} & =\frac{\partial a}{\partial a_i} \left(\frac{\beta}{\alpha+\beta} \right)\\
\frac{\partial \log(\mathbb{P}(T=1))}{\partial b_i} & =-\frac{\partial b}{\partial b_i} \left(\frac{\beta}{\alpha+\beta}\right)
\end{aligned}
\end{equation*}

These derivatives expand as follows:
\begin{equation*}
\begin{aligned}
\frac{\partial \log(\mathbb{P}(T=t))}{\partial a_i} & = \frac{\partial \log(\mathbb{P}(T=t-1))}{\partial a_i} - \frac{\partial a}{\partial a_i}\left( \frac{\alpha}{\alpha + \beta +t-1} \right) \\
\frac{\partial \log(\mathbb{P}(T=t))}{\partial b_i} & = \frac{\partial \log(\mathbb{P}(T=t-1))}{\partial b_i} \\
& + \frac{\partial b}{\partial b_i}\left(\frac{(\alpha+1)\beta}{(\beta+t-2)(\alpha + \beta+t-1)} \right)
\end{aligned}
\end{equation*}

We can get a similar recursion for the survival function:
\begin{equation*}
\begin{aligned}
\frac{\partial \log(\mathbb{P}(T > 1))}{\partial a_i} & =-\frac{\partial a}{\partial a_i}\left(\frac{\alpha}{\alpha + \beta}\right) \\
\frac{\partial \log(\mathbb{P}(T > 1))}{\partial b_i} & =\frac{\partial b}{\partial b_i}\left(\frac{\alpha}{\alpha + \beta}\right)
\end{aligned}
\end{equation*}
\begin{equation*}
\begin{aligned}
\frac{\partial \log(\mathbb{P}(T > t))}{\partial a_i} & =\frac{\partial \log(\mathbb{P}(T > t-1))}{\partial a_i} -\frac{\partial a}{\partial a_i}\left(\frac{\alpha}{\alpha + \beta+t-1}\right) \\
\frac{\partial \log(\mathbb{P}(T > t))}{\partial b_i} & = \frac{\partial \log(\mathbb{P}(T > t-1))}{\partial b_i} \\
& + \frac{\partial b}{\partial b_i}\left(\frac{\alpha\beta}{(\beta+t-1)(\alpha + \beta +t-1)} \right)
\end{aligned}
\end{equation*}

\subsection{Diagonal of the Hessian}
We obtain the second derivatives for the Hessian as follows:
\label{subsection:hessian}
\begin{equation*}
\begin{aligned}
\frac{\partial^2 \log(\mathbb{P}(T=1))}{\partial a_i^2} & = \frac{\partial^2 a}{\partial a_i^2}\left(\frac{\beta}{\alpha+\beta}\right)-\left(\frac{\partial a}{\partial a_i}\right)^2 \left(\frac{\alpha\beta}{(\alpha + \beta)^2} \right)\\
\frac{\partial^2 \log(\mathbb{P}(T=1))}{\partial b_i^2} & =-\frac{\partial^2 b}{\partial b_i^2}\left(\frac{\beta}{\alpha+\beta}\right)-\left(\frac{\partial b}{\partial b_i}\right)^2 \left(\frac{\alpha\beta}{(\alpha + \beta)^2} \right)
\end{aligned}
\end{equation*}
\begin{equation*}
\begin{aligned}
\frac{\partial^2 \log(\mathbb{P}(T=t))}{\partial a_i^2} & =\frac{\partial^2 \log(\mathbb{P}(T=t-1))}{\partial a_i^2} - \left(\frac{\partial^2 a}{\partial a_i^2}\right)\left( \frac{\alpha}{\alpha + \beta +t-1} \right) \\
& - \left(\frac{\partial a}{\partial a_i}\right)^2\alpha\left( \frac{\beta+t-1}{(\alpha + \beta +t-1)^2} \right) \\
\frac{\partial^2 \log(\mathbb{P}(T=t))}{\partial b_i^2} & = \frac{\partial^2 \log(\mathbb{P}(T=t-1))}{\partial b_i^2} \\
& + \left(\frac{\partial^2 b}{\partial b_i^2}\right)\left(\frac{(\alpha+1)\beta}{(\beta+t-2)(\alpha + \beta+t-1)} \right) \\
& + \left(\frac{\partial b}{\partial b_i}\right)^2\beta\left((\alpha+1)\frac{\beta^2 - (t-2)(\alpha + t - 1)}{\left(\beta+t-2\right)^2\left(\alpha+\beta+t-1\right)^2}  \right)
\end{aligned}
\end{equation*}

The survival counterparts to the above terms are also readily computed as follows:
\begin{equation*}
\begin{aligned}
\frac{\partial^2 \log(\mathbb{P}(T > 1))}{\partial a_i^2} & =- \frac{\partial^2 a}{\partial a_i^2}\left(\frac{\alpha}{\alpha+\beta}\right)-\left(\frac{\partial a}{\partial a_i}\right)^2 \left(\frac{\alpha\beta}{(\alpha + \beta)^2} \right)\\
\frac{\partial^2 \log(\mathbb{P}(T > 1))}{\partial b_i^2} & =\frac{\partial^2 b}{\partial b_i^2}\left(\frac{\alpha}{\alpha+\beta}\right)-\left(\frac{\partial b}{\partial b_i}\right)^2 \left(\frac{\alpha\beta}{(\alpha + \beta)^2} \right).
\end{aligned}
\end{equation*}
\begin{equation*}
\begin{aligned}
\frac{\partial^2 \log(\mathbb{P}(T > t))}{\partial a_i^2} & =\frac{\partial^2 \log(\mathbb{P}(T > t-1))}{\partial a_i^2} - \left(\frac{\partial^2 a}{\partial a_i^2}\right)\left( \frac{\alpha}{\alpha + \beta +t-1} \right) \\
& - \left(\frac{\partial a}{\partial a_i}\right)^2\alpha\left( \frac{\beta+t-1}{(\alpha + \beta +t-1)^2} \right) \\
\frac{\partial^2 \log(\mathbb{P}(T > t))}{\partial b_i^2} & = \frac{\partial^2 \log(\mathbb{P}(T > t-1))}{\partial b_i^2}  \\
& + \left(\frac{\partial^2 b}{\partial b_i^2}\right)\left(\frac{\alpha \beta}{(\beta+t-1)(\alpha+\beta+t-1)}\right) \\
& + \left(\frac{\partial b}{\partial b_i}\right)^2\beta\left(  \alpha\frac{\beta^2 - (t-1)(\alpha + t-1) }{\left(\beta+t-1\right)^2\left(\alpha+\beta+t-1\right)^2} \right)
\end{aligned}
\end{equation*}

\section{Alternative Derivation}
Another intuitive derivation of the single-step beta-logistic is obtained by starting from the likelihood for a logistic model and modeling the probabilities with a beta distribution:
\begin{equation*}
\begin{aligned}
    L & = \prod_{i} \mathbb{P}(y_i=1|\alpha,\beta)^{y_i} (1-\mathbb{P}(y_i=1|\alpha,\beta))^{y_i-1} \\
      & = \prod_{\forall y_i=1} \mathbb{P}(y_i=1| \alpha,\beta)\prod_{\forall y_i=0} (1-\mathbb{P}(y_i=1| \alpha,\beta)) \\
      & = \prod_{uncensored} \mathbb{P}(T=1| \alpha,\beta)\prod_{censored} \mathbb{P}( t >=1 | \alpha,\beta).
\end{aligned}
\end{equation*}
This is exactly the survival likelihood for a 1 step beta logistic model.
\end{appendix}
\clearpage
\onecolumn

\section{Reproducibility}

We include simple python implementations of the gradient callbacks that can be passed to XgBoost or lightGBM. Note that efficient implementations of these callbacks in C++ are possible and yield orders of magnitude speedups.  

\begin{lstlisting}[language=Python]
def grad_BL(alpha,beta,t,is_censored):
    """
    This function computes the gradient of the beta logistic objective.
    Since it is vectorized in practice for performance reasons, here we write
    the non-vectorized version for readability:
    """
    N = len(alpha)
    g = np.zeros((N, 2))
    for j in range(0,N):
        if (not is_censored[j]):
            #failed
            g[j,0] = beta[j]/(alpha[j] + beta[j])
            g[j,1] = -g[j,0]
            for i in range(2,int(t[j] + 1)):
                g[j,0] += -(alpha[j]/(alpha[j] + beta[j] + i - 1))
                g[j,1] += beta[j]/(beta[j] + i - 2) - beta[j]/(alpha[j] + beta[j] + i - 1)
        else:
            #survived
            g[j,:] = -alpha[j]/(beta[j] + alpha[j])
            g[j,1] = -g[j,0]
            for i in range(2,int(t[j] + 1)):
                g[j,0] += -(alpha[j]/(alpha[j] + beta[j] + i - 1))
                g[j,1] += beta[j]/(beta[j] + i - 1) - beta[j]/(alpha[j] + beta[j] + i - 1)
    return g

def hess_BL(alpha,beta,t,is_censored):
    """
    This function computes the diagonal of the Hessian of the beta logistic objective.
    """
    N = len(alpha)
    h = np.zeros((N,2))
    for j in range(0,N):
        h[j:] = -alpha[j]*beta[j]/((alpha[j] + beta[j])**2)
        if (not is_censored[j]):
            #failed
            for i in range(2,int(t[j] + 1)):
                h[j,0] += -alpha[j]*((beta[j] + i - 1)/(alpha[j] + beta[j] + i - 1)**2)
                d = (beta[j] + i - 2)**2)*(alpha[j] + beta[j] + i - 1)**2)
                h[j,1] += beta[j]*((alpha[j]+1)*(beta[j]**2-(i-2)*(alpha[j]+i-1)/d
        else:
            #survived
            for i in range(2,int(t[j] + 1)):
                h[j,0] += -alpha[j]*((beta[j] + i - 1)/(alpha[j] + beta[j] + i - 1)**2)
                d = (beta[j] + i - 2)**2)*(alpha[j] + beta[j] + i - 1)**2)
                h[j,1] += beta[j]*((alpha[j])*(beta[j]**2-(i-1)*(alpha[j]+i-1)/d

    h.shape = (N*2)
    return h
    
def likelihood_BL(alpha,beta,t,is_censored):
    """
    This function computes beta logistic objective (likelihood : higher = better)
    Since it is heavily vectorized in practice for performance reasons, we write 
    here the non-vectorized version for readability:
    """
    p = alpha / (alpha + beta)
    s = 1 - p
    for j in range(0,len(alpha)):
        for i in range(2,int(t[j]+1)):
            p[j] = p[j] * (beta[j] + i - 2)/(alpha[j] + beta[j] + i - 1)
            s[j] = s[j] - p[j]
    return p * (1.0 - is_censored) + s * is_censored


\end{lstlisting}

\end{document}